\documentclass[a4paper,11pt]{article}
\usepackage[a4paper, margin=1in]{geometry}
\usepackage[utf8]{inputenc}
\usepackage[T1]{fontenc}
\usepackage{graphicx} 

\usepackage{amsmath}
\usepackage{amsfonts}
\usepackage{amsthm}
\usepackage{xcolor}
\usepackage{url}
\usepackage{mathtools}
\usepackage{natbib} 
\definecolor[named]{urlblue}{cmyk}{1,0.58,0,0.21}
\usepackage[pagebackref,
 pdfa,
hidelinks,
pdftex, 
pdfdisplaydoctitle,
pdfpagelabels,
pdfauthor={Noah Daniëls and Floris Geerts},
pdftitle={VC Dimension of 1-dimensional GNNs},
pdfsubject={},
pdfkeywords={},
pdfproducer={Latex with the hyperref package},
pdfcreator={pdflatex}
]{hyperref}

\hypersetup{
 breaklinks=true,
 colorlinks=true,
 citecolor=purple!70!blue!60!black,
 linkcolor=purple!70!blue!90!black,
 urlcolor=urlblue,
 pdflang={en},
 pdftitle={VC dimension of 1-dimensional GNNs},
 pdfauthor={Noah Daniëls and Floris Geerts}
}
\DeclareMathOperator{\vcdim}{VC-dim}
\DeclareMathOperator{\gnn}{GNN}
\newcommand{\set}[1]{\bigl\{#1\bigr\}}
\newcommand{\mset}[1]{\bigl\{\!\!\bigl\{#1\bigl\}\!\!\bigl\}}

\newcommand{\Nb}{\mathbb{N}}
\newtheorem{theorem}{Theorem}[section]
\newtheorem{proposition}{Proposition}[section]
\newtheorem{lemma}{Lemma}[section]
\theoremstyle{definition}

\newcommand{\REL}{\mathsf{relabel}}
\newcommand{\bx}{\mathbf{x}}
\newcommand{\Gc}{\mathcal{G}}
\newcommand{\Xc}{\mathcal{X}}
\newcommand{\Cc}{\mathcal{C}}
\newcommand{\Fc}{\mathcal{F}}
\newcommand{\Oc}{\mathcal{O}}
\newcommand{\upd}{\mathsf{upd}}
\newcommand{\agg}{\mathsf{agg}}
\newcommand{\readout}{\mathsf{readout}}
\newcommand{\bh}{\mathbf{h}}
\newcommand{\Rb}{\mathbb{R}}
\theoremstyle{remark}

\newcommand{\wlone}{$1$\textrm{-}\textsf{WL}}

\title{A note on the VC dimension of $1$-dimensional GNNs}
\author{
Noah Dani\"els
\and
Floris Geerts
}
\date{June 2024}

\begin{document}

\maketitle

\begin{abstract}
Graph Neural Networks (GNNs) have become an essential tool for analyzing graph-structured data, leveraging their ability to capture complex relational information. While the expressivity of GNNs, particularly their equivalence to the Weisfeiler-Leman ($\wlone$) isomorphism test, has been well-documented, understanding their generalization capabilities remains critical. This paper focuses on the generalization of GNNs by investigating their Vapnik–Chervonenkis (VC) dimension. We extend previous results to demonstrate that $1$-dimensional GNNs with a \emph{single parameter} have an infinite VC dimension for unbounded graphs. Furthermore, we show that this also holds for GNNs using analytic non-polynomial activation functions, including the $1$-dimensional GNNs that were recently shown to be as expressive as the $\wlone$ test. These results suggest inherent limitations in the generalization ability of even the most simple GNNs, when viewed from the VC dimension perspective.
\end{abstract}

\section{Introduction}
Graph Neural Networks (GNNs) have emerged as a powerful tool for analyzing and processing graph-structured data, which is common in domains like social networks and biological or chemical networks. Unlike traditional neural networks, GNNs naturally incorporate complex relational information, enabling them to learn rich representations of vertices, edges, and entire graphs. GNNs have achieved state-of-the-art performance in many applications, demonstrating their utility in solving real-world problems involving relational data.

The expressive power of GNNs, which refers to their ability to capture and distinguish different graph structures and properties, is a fundamental aspect that determines their effectiveness in various applications. Understanding the expressivity of GNNs is crucial because it directly impacts their capacity to model complex relationships and make accurate predictions. The expressivity of GNNs was linked to the classical Weisfeiler-Leman ($\wlone$) isomorphism test~\citep{Wei+1968} by \citet{expressive1} and \citet{expressive2}. They  independently proved that for any graph, a GNN can be constructed which perfectly simulates $\wlone$, and at the same time, that any GNN is bounded by $\wlone$ in expressive power. Unfortunately, the GNNs constructed in both proofs are non-uniform in the sense that the width of the GNNs is required to grow with the size of the graphs. \citet{expressive3} improved the construction to only require logarithmic growth in the size of the graphs but using a randomized algorithm with high probability guarantees. Very recently, \citet{expressive4} and \citet{expressive5} showed that GNNs of width one are already sufficient if analytic non-polynomial activation functions are used. Moreover, these results show the existence of GNNs which can simulate $\wlone$ for graphs of any size, answering an open question posed by \citet{logic} about whether such a uniform simulation indeed exists.

Besides the expressive power, another important research direction is the ability of GNNs to generalize to unseen data. This is a critical factor for their success in real-world applications, as they determine how well a model trained on a specific dataset can perform in practice on unseen data. The Vapnik–Chervonenkis (VC) dimension~\citep{vc1} is a key concept in learning theory which provides insight into a model’s ability to generalize (see \citet{vc2} for a recent account). A lower VC dimension typically indicates better generalization properties.

A trade-off is expected, however, between a model's complexity and its ability to generalize. A complex model is usually more expressive but often also has poor generalization abilities as it can more easily capture noise specific to the training data. As the GNNs considered by \citet{expressive5} are very simple but still as expressive as $\wlone$, it begs the question how their VC dimension compares to that of larger GNNs. We show that the VC dimension of even the most simple and small GNNs (including those of \citet{expressive5}) is infinite for unbounded graphs. This is in line with previous results by \citet{generalization1} which show that for GNNs using piecewise linear activation functions and whose width and depth is at least two, the VC dimension is infinite. We extend these results to GNNs of width and depth one, and to such GNNs using analytic non-polynomial activation functions and only a single parameter.

\smallskip
\noindent
To summarize, our contributions are the following:
\begin{itemize}
    \item We provide an alternative proof for a result from \cite{generalization1} stating that the VC dimension of GNNs with piecewise linear activation functions is infinite if the depth and width is at least two. Our proof strengthens this result to GNNs whose width and depth is one.
    \item We prove a similar result for GNNs restricted to analytic non-polynomial activation functions. In particular, we apply this to the GNNs described by \cite{expressive5} to show that their generalization ability in terms of VC dimension for unbounded graphs is not improved.
\end{itemize}

\section{Preliminaries}
We first define $\Nb\coloneqq\set{0,1,2,\dots}$ 
and $\Nb_{>0}\coloneqq\set{1,2,\ldots}$. Then, for $n\in\Nb_{>0}$, let $[n]\coloneqq\set{1,2\dots,n}\subset\Nb_{>0}$. We denote multisets, that is,
the generalisation of sets allowing elements to occur multiple times, as $\mset{\dots}$. We use bold letters to denote tuples and the entries of a $k$-tuple $\bx$ are written as $x_1,x_2,\dots x_k$.

\subsection{Graphs}
We consider the set $\Gc$ of finite, undirected, simple, vertex-labelled graphs. A \emph{graph} $G \in \Gc$ is a triple $(V(G), E(G), \ell)$ consisting of a finite set of \emph{vertices} (or \emph{nodes}) $V(G)$, a binary \emph{edge} relation $E(G) \subseteq V(G) \times V(G)$ which is symmetric and irreflexive, and a \emph{vertex-label} (or \emph{vertex-coloring}) function $\ell: V(G) \to \Nb$. The value $\ell(v)$ is the \emph{label} (or \emph{color}) of a vertex $v \in V(G)$. The \emph{order} of a graph $G$ is the number $|V(G)|$ of vertices, typically denoted by $n$. Additionally, the \emph{neighborhood} of a vertex $v$ is denoted by $N_G(v) \coloneqq \{ u \in V(G) \mid (u,v) \in E(G) \}$, and the \emph{degree} of $v$ is given by $d_G(v) \coloneqq |N_G(v)|$. A graph is called \emph{$k$-regular} if every vertex in $G$ has degree $k$.

Let $\ell: V(G) \to \Nb$ and $\ell': V(G) \to \Nb$ be two vertex-coloring functions. We say that $\ell$ and $\ell'$ are \emph{equivalent} if, for all $v, w \in V(G)$, $\ell(v) = \ell(w)$ if and only if $\ell'(v) = \ell'(w)$.

Two graphs $G$ and $H$ are said to be \emph{isomorphic}, denoted $G \simeq H$, if there exists a bijection (an \emph{isomorphism}) $\phi: V(G) \to V(H)$ that preserves both the edge relation and vertex labels, that is, $(u, v) \in E(G)$ if and only if $\bigl(\phi(u), \phi(v)\bigr) \in E(H)$, and $\ell(v) = \ell(\phi(v))$ for all $v \in V(G)$. 
A function $\xi$ defined on graphs is said to be \emph{isomorphism-invariant} if $\xi(G) = \xi(H)$ for all isomorphic graphs $G$ and $H$.

\subsection{Graph neural networks}
A Graph Neural Network (GNN) iteratively computes vectorial representations of the vertices of a labeled graph by aggregating local neighborhood information. Initially each vertex $v$ is represented by a vector $\bh_v^{(0)} \in \Rb^{d^{(0)}}$ which is consistent with the graph's labeling function. That is, $\bh_v^{(0)}=\bh_u^{(0)}$ if $\ell(u)=\ell(v)$. During iteration $t>0$, the GNN computes a $d^{(t)}$-dimensional vertex embedding 
$$
\bh_v^{(t)} \coloneq \upd^{(t)}\biggl(\bh_v^{(t-1)},\agg^{(t)}\Bigl(\mset{\bh_u^{(t-1)} \mathrel{\big|} u \in N_G(v)}\Bigr)\biggr) \in \Rb^{d^{(t)}},
$$
where $\upd^{(t)}$ and $\agg^{(t)}$ are arbitrary parameterized update and aggregation functions. These are typically differentiable functions such as neural networks to facilitate optimization for specific tasks. From the embedding computed in the final iteration $L$, a graph invariant $\xi$ is computed in a final readout step
$$
\xi(G)\coloneq \readout\Bigl(\mset{\bh_v^{(L)}\mathrel{\big|} v \in V(G)}\Bigr) \in \Rb,
$$
where $\readout$ is also a parameterized function.

We refer to the maximum size of the vertex embeddings $d\coloneq \max_t\left(d^{(t)}\right)$ as the \emph{width} (or dimensionality) of the GNN and to the number of iterations (or layers) $L$ before the readout as the \emph{depth}. The number of bits needed to specify the parameters of all update and aggregation functions and the readout function is the bitwidth of the GNN.

We denote the class of GNNs of depth $L$ and width at most $d$ as $\gnn(d,L)$. We further considers two subsets of this class where the update, aggregation and readout functions are restricted. For both subsets the aggregation functions is the same for all iterations and simply sums the embeddings of the neighboring vertices. For the subset $\gnn_{\textsl{pl}}(d,L)$ we restrict the update and readout functions to single layer perceptrons using piecewise linear activation functions. Similarly, for the subset $\gnn_{\textsl{anp}}(d,L)$ the update and readout functions are restricted to single layer perceptrons using analytic non-polynomial activation functions.

\subsection{The Weisfeiler--Leman algorithm}

The \wlone{} or color refinement is a well-studied heuristic for the graph isomorphism problem, originally proposed by~\citet{Wei+1968}. Let $G = (V(G),E(G),\ell)$ be a labeled graph. In each iteration, $t > 0$, the \wlone{} computes a vertex coloring $C^1_t \colon V(G) \to \Nb$, depending on the coloring of the neighbors. That is, in iteration $t>0$, we set
\begin{equation*}
	C^1_t(v) \coloneqq \REL\biggl(\Bigl(C^1_{t-1}(v),\mset{ C^1_{t-1}(u) \mid u \in N_G(v)}\Bigr)\biggr),
\end{equation*}
for all vertices $v$ in $V(G)$,
where $\REL$ injectively maps the above pair to a unique natural number, which has not been used in previous iterations. In iteration $0$, the coloring $C^1_{0}\coloneqq \ell$. To test if two graphs $G$ and $H$ are non-isomorphic, we run the above algorithm in ``parallel'' on both graphs. If the two graphs have a different number of vertices colored $c$ in $\Nb$ at some iteration, the \wlone{} \emph{distinguishes} the graphs as non-isomorphic. Moreover, if the number of colors between two iterations, $t$ and $(t+1)$, does not change, that is, the cardinalities of the images of $C^1_{t}$ and $C^1_{i+t}$ are equal, or, equivalently,
\begin{equation*}
	C^1_{t}(v) = C^1_{t}(w) \iff C^1_{t+1}(v) = C^1_{t+1}(w),
\end{equation*}
for all vertices $v$ and $w$ in $V(G)$, the algorithm terminates. For such $t$, we define the \emph{stable coloring}
$C^1_{\infty}(v) = C^1_t(v)$, for $v$ in $V(G)$. The stable coloring is reached after at most $\max \{ |V(G)|,|V(H)| \}$ iterations~\citep{Gro2017}. 

\subsection{Generalisation ability and VC dimension}
A class $\Cc$ of binary classifiers is said to \emph{shatter} a set $X$ of inputs if every possible classification of elements in $X$ is realizable. In other words, for every partitioning of $X$ into disjoint $X^+$ and $X^-$, there is a classifier $h$ in $\Cc$ such that $h$ assigns a positive label to $x$ if and only if $x$ is in $X^+$. The \emph{VC dimension} $\vcdim_\Xc(\Cc)$ of a class $\Cc$ for a set of inputs $\Xc$ is the size of the largest subset $X \subseteq \Xc$ which the class can shatter.

Specifically for GNNs, we write that $\vcdim_\Xc(\Cc)$ is the maximal number $m$ of graphs $\set{G_1, G_2, \dots, G_m} \subseteq \Xc$ such that for any $\bx\in\set{0,1}^m$ there exists a GNN in $\Cc$ such that for the graph invariant $\xi_\textbf{x}$ computed by this GNN $\xi_\textbf{x}(G_i) \geq t$ if and only if $x_i=1$ for a fixed threshold $t$.

The VC dimension is a key concept in learning theory which provides insight into the model's ability to generalize from training data to unseen data. The VC dimension can be used to bound with high probability the difference between the empirical error and the true error. We refer to \cite{vc1} and \cite{vc2} for details.
A class with low VC dimension is thus able to generalize better, but this usually comes at the cost of reduced expressivity.

\section{Formal statement of previous and new results}
Given these preliminaries, we can now state the relevant previous known results about the expressivity and generalisation ability of GNNs. We start with a result relating the expressive power of GNNs to the $\wlone$ test. The following theorem states that GNNs can perfectly simulate $\wlone$ if the width of the GNN is allowed to grow with $n$.

\begin{theorem} \citep{expressive1, expressive2} \label{th:expressive}
    Let $n \in \Nb_{>0}$. Then there exists a GNN in $\gnn_{\textsl{pl}}(\Oc(n),n)$ such that for all $t\in[n]$ the vertex embedding computed in the $t$-th layer is equivalent to the coloring produced by the $\wlone$ test after $t$ iterations on all graphs of order at most $n$. 
\end{theorem}

\cite{expressive3} showed that this can be improved exponentially by letting the width grow as $\Oc(\log n)$. It was then shown by \citet{sammy1} that in the case of polynomial activations, a uniform solution, that is, where the size of the GNN does not need to grow with $n$, is impossible. Whether this holds for other activation functions was recently answered by \citet{expressive5} who proved that for analytic non-polynomial activation functions, $1$-dimensional GNNs can simulate $\wlone$ for any graph.

\begin{theorem} \citep{expressive5}
    Let $a:\Rb\rightarrow\Rb$ be any analytic non-polynomial function and $L \in \Nb_{>0}$. Then there exists a $1$-dimensional GNN in $\gnn_{\textsl{anp}}(1,L)$ such that for all $t\in[L]$ the vertex embedding computed in the $t$-th layer is equivalent to the coloring produced by the $\wlone$ test after $t$ iterations on all graphs.
\end{theorem}

From these results it seems that the activation function plays an important role in determining the expressive power of GNNs. But what about their generalisation ability? In the case of piecewise linear activation functions, \citet{generalization1} showed that for unbounded graphs, that is, for graphs of arbitrary order, the VC dimension infinite.

\begin{theorem} \citep{generalization1} \label{th:gnn_pp_infinity1}
For all $d, L$ at least two, it holds that
the VC dimension of $\gnn_{\textsl{pl}}(d,L)$ is unbounded. That is,
$\vcdim_{\Gc}({\gnn_{\textsl{pl}}(d,L)})=\infty$.
\end{theorem}

The restriction that $d$ and $L$ should be at least two follows from the construction of the lower bound given by \citet{generalization1}. Indeed,
the GNNs used in that proof have $d=L=2$.  

As our first contribution, we present an alternate proof of Theorem~\ref{th:gnn_pp_infinity1}  showing that for piecewise polynomial activation functions, a width and depth of 1 is also sufficient.

\begin{theorem}[Stronger version of Theorem~\ref{th:gnn_pp_infinity1}] \label{th:gnn_pp_infinity2}
For all $d$, $L$ in $\Nb_{>0}$, it holds that
the VC dimension of $\gnn_{\textsl{pl}}(d,L)$ is unbounded. That is, it holds that  $\vcdim_{\Gc}({\gnn_{\textsl{pl}}(d,L)})=\infty$.
\end{theorem}

Furthermore, as our second contribution, we show that this result also holds when we are restricted to analytic non-polynomial activation functions. Furthermore, we do so using $1$-dimensional GNNs and just a \emph{single parameter}, making this result applicable to the class of GNNs considered by \cite{expressive5}.

\begin{theorem} \label{th:main}
For all $d$, $L$ in $\Nb_{>0}$, it holds that
the VC dimension of $\gnn_{\textsl{pl}}(d,L)$ is unbounded. That is, $\vcdim_{\Gc}({\gnn_{\textsl{anp}}(d,L)})=\infty$, even for architectures with only a single parameter.
\end{theorem}

\section{Piecewise polynomial activation functions} \label{sec:proof1}
We start by proving Theorem~\ref{th:gnn_pp_infinity2}, which as just mentioned, is a slightly stronger version of Theorem~\ref{th:gnn_pp_infinity1}. We follow the same outline as in the original proof by first proving the following proposition, from which the theorem immediately follows as we let the bitlength tend to infinity.

\begin{proposition}[Analog of Proposition 3.5 in \cite{generalization1}]\label{prop:pl-bit}
    There exists family $\Fc_b$ of $1$-layer GNNs of width 1 and bitlength $\Oc(b)$ using piecewise linear activation functions such that $\vcdim_{\Gc}(\Fc_b)=b$.
\end{proposition}

\begin{proof}
The upper bound follows from standard learning-theoretic results for finite classes of GNNs where an upper bound in known logarithmic in the size of the class. The size of $\Fc_b$ is finite due to the bounded bitlength, feature width and layer depth assumptions. More precisely,
the size of $\Fc_b$ is exponential in $b$, from which the result follows.

For the lower bound, we define the GNNs of $\Fc_b$ and construct a set of b graphs that are shattered by $\Fc_b$. The GNNs are parameterized by $\bx\in\set{0,1}^b$ using the parameter $\gamma_\bx=\sum_{k=1}^{b}{x_k\cdot2^{k-1}}$ taking on integer values in $[0,2^b-1]$. The graph invariant computed by the GNN is defined as follows:
\begin{align*}
    h^{(0)}_{\bx,v} &= 1, \\
    h^{(1)}_{\bx,v} &= f\left(\gamma_\bx h^{(0)}_{\bx,v} -\sum_{u\in N_G(v)}{h^{(0)}_{\bx,u}}\right)\text{, and} \\
    \xi_\bx(G) & =\sum_{v\in V(G)}{h^{(1)}_{\bx,v}}.
\end{align*}
We use the following piecewise linear activation function, independent of $b$ or $\bx$:
\begin{align*}
    f(x) &=\begin{cases}
        0 & \text{if } x \leq -3; \\
        x+3 & \text{if } -3 < x \leq -2; \\
        -x-1 & \text{if } -2 < x < -1;\text{ or} \\
        0 & \text{if }  x \geq -1. \\
    \end{cases}
\end{align*}
As can be observed, the GNNs are indeed of width and depth one, and we have only used two parameters, the constant $-1$ and the variable parameter $\gamma_\bx$ for $\bx\in\{0,1\}^b$ of bitlength $b$.

We next define the graph collection. For $j=1,2,\dots,b$, let graph $G_j$ be constructed by starting with $2^{b-1}$ vertices denoted $s_0,\dots,\allowbreak s_{2^{b-1}-1}$, each of which will be connected to a number of dummy vertices to ensure specific vertex degrees.

Let $\boldsymbol{\iota}_i=(\iota_{i,1},\iota_{i,2},\dots,\iota_{i,b-1})\in\{0,1\}^{b-1}$ be the binary representation of an integer $i$ in $[0,2^{b-1}-1]$, that is, $i=\sum_{k=1}^{b-1}\iota_{i,k}\cdot2^{k-1}$. The desired degree $d_i$ of vertex $s_i$ in graph $G_j$ will be $d_i=q_{i,j}+2$ where $q_{i,j}$ is given by:
$$
q_{i,j} = \sum_{k=1}^{j-1}\iota_{i,k}2^{k-1}  + 2^{j-1} + \sum_{k=j}^{b-1}\iota_{i,k}2^{k}\in [0,2^b-1].
$$
In other words, the binary representation of $q_{i,j}$ is $(\iota_{i,1}\iota_{i,2}\dots\iota_{i,j-1}1\iota_{i,j}\dots\iota_{i,b-1})$ meaning that $q_{i,j}$'s $j$-th bit will be 1. We achieve these vertex degrees by adding sufficiently many dummy vertices, all of which have degree 1.

Furthermore, by this construction for any integer in $[0, 2^b-1]$ whose $j$-th bit is 1, there will be exactly one $i$ such that $q_{i,j}$ is this integer.
To see this, we show that the sets $A_j = \set{a \in [0, 2^b-1] \mathrel{\big|} \text{$a$'s $j$-th bit is 1}}$ and $Q_j = \set{q_{i,j} \mathrel{\big|} i \in [0, 2^{b-1}-1]}$ are equal. 
\begin{itemize}
    \item $Q_j \subseteq A_j$ because every $q_{i,j}$ is an integer in $[0, 2^{b}-1]$ (because their bit length is $b$) whose $j$-th bit is 1 (by construction).
    \item $|A_j| = 2^{b-1}$ because there are $2^b$ integers in $[0, 2^b-1]$ and half of those have their $j$-th bit equal to 1.
    \item $|Q_j| = 2^{b-1}$ because there are are $2^{b-1}$ integers $i$ in $[0, 2^{b-1}-1]$ and each $i$ corresponds to a different $q_{i,j}$ ($i$'s bit representation is contained within the bit representation of $q_{i,j}$).
\end{itemize}

We finally show that the embeddings of the GNN $\xi_\bx$ on $G_1,\ldots,G_b$ 
yields $x_1,\ldots,x_b$, respectively.
This shows that this set of $b$ graphs
can be shattered, as desired.
Indeed, denoting dummy vertices with $\delta$:
\begin{align*}
    h^{(1)}_{\bx,\delta}   &= f(\underbrace{\gamma_\bx-1}_{\geq-1}) = 0, \text{ and} \\
    h^{(1)}_{\bx,s_i} &= f(\gamma_\bx-d_i) = f(\gamma_\bx-q_{i,j}-2) = \begin{cases}
        1 & \text{if } \gamma_\bx = q_{i,j};\text{ and} \\
        0 & \text{else.}
    \end{cases}
\end{align*}
Now if $x_j=1$, then $\gamma_\bx$ will be an integer in $[0,2^b-1]$ whose $j$-th bit is 1. This means there is exactly one $i$ such that $q_{i,j}$ is equal to $\gamma_\bx$ and thus there is exactly one vertex $s_i$ whose value is 1. On the other hand, if $x_j=0$, then no such $i$ exist because all $q_{i,j}$ have their $j$-th bit equal to 1 and all vertices will have a weight of zero. Summing all the vertex weights to compute the final graph embedding results in $\xi_\bx(G_j)=x_j$ proving that $b$ graphs can be shattered by a one layer GNN using parameters $\gamma_\bx$ and $-1$ and piecewise linear activation functions.
\end{proof}

\section{Analytic non-polynomial activation functions} \label{sec:proof2}

We already mentioned that $1$-dimensional GNNs with analytic non-polynomial activation functions, in contrast to those with polynomial activation functions, can simulate $\wlone$ for any graph \citep{sammy1,expressive5}. This increase in expressive power, however, does not translate to better generalisation in terms of VC dimension. Indeed, in this section we present a proof of Theorem~\ref{th:main}, showing that these GNNs have again unbounded VC dimension. Furthermore, we show that even a single parameter in the GNNs suffices.

We first need the following lemma, showing that the
sine function can be used to extract bits from its argument.

\begin{lemma} \label{lemma:sin}
Let $b\in\Nb_{>0}$, $\bx=(x_1,x_2,\dots x_n)\in\{0,1\}^b$ and $\gamma_\bx= \sum\nolimits_{k=1}^{b}\frac{x_k}{4^k}$.
Then $\gamma_\bx\in[0,\frac{1}{3}]$ and for $j>0$
$$
\sin\bigl(4^j\frac{\pi}{2}\gamma_\bx\bigr) = 
\begin{cases}
 = 0 & \text{if } j > b; \\
 \geq \frac{2}{3} & \text{if } x_j = 1; \text{ and} \\
 \leq \frac{1}{2} & \text{if } x_j = 0.
\end{cases}
$$

\end{lemma}
\begin{proof}
We first observe that irrespective of the choice of $\bx$, for $i \leq b$:
\begin{equation} \label{eq:1}
\sum_{k=1}^{i}\dfrac{x_k}{4^k} \leq 
\sum_{k=1}^{\infty}\dfrac{1}{4^k}=\dfrac{1}{3},
\end{equation}
as desired. 

For the second part, we first multiply $\gamma_\bx$ with $4^{j}\frac{\pi}{2}$ for $j>b$:
$$
    4^{j}\dfrac{\pi}{2}\gamma_\bx = 4\frac{\pi}{2}\sum_{i=1}^{b}\underbrace{4^{j-1}4^{-i}x_i}_{\in\Nb} = 2k\pi \quad\text{with }k\in\Nb
$$
and find $\sin\bigl(4^{j}\dfrac{\pi}{2}\gamma_\bx\bigr)=\sin(2k\pi)=0$.
Next, for $j=1,\dots,b$ we find the following:
\begin{align*}
    4^{j}\dfrac{\pi}{2}\gamma_\bx & = \frac{\pi}{2}\sum_{i=1}^{b}4^j4^{-i}x_i \\
    & = \frac{\pi}{2}\sum_{i=1}^{j-1}{4^{j-i}x_i} + \frac{\pi}{2}x_j + \frac{\pi}{2}\sum_{i=j+1}^{b}{4^{j-i}x_i} \\
    & = 2\pi\sum_{i=1}^{j-1}\underbrace{4^{j-i-1}{x_i}}_{\in\Nb} + \frac{\pi}{2}x_j + \frac{\pi}{2}\sum_{i=1}^{b-j}{4^{-i}x_{i+j}} \\
    & = 2k\pi + \frac{\pi}{2}x_j + \epsilon \quad\text{with }\epsilon \leq \frac{\pi}{6},k\in\Nb
\end{align*}
where the last step uses \eqref{eq:1} to bound $\epsilon$. After applying the $\sin$ function, we obtain
\begin{align*}
\sin\bigl(4^j\dfrac{\pi}{2}\gamma_\bx\bigr) 
& = \sin\bigl(2k\pi + \frac{\pi}{2}x_j + \epsilon\bigr) \\
&= \sin\bigl(\frac{\pi}{2}x_j + \epsilon\bigr) \\
&=\begin{cases}
\geq \sin\bigl(\frac{\pi}{2} + \frac{\pi}{6}\bigr) \approx 0.866 & \text{if } x_j = 1; \text{ and} \\
\leq \sin\bigl(\frac{\pi}{6}\bigr) = 0.5 & \text{if } x_j = 0.
\end{cases}
\end{align*}
as desired.
\end{proof}

We next use Lemma~\ref{lemma:sin} to show that there is a finite class $\Fc_b$ of GNNs in $\gnn_{\textsl{anp}}(1,1)$, each GNN consisting of one layer, of dimension one, and only using  single parameters of bitlength $\Oc(b)$, which can shatter $b$ graphs. Furthermore, the GNNs in $\Fc_b$ use a modified sine function as activation function. We emphasize that the bitlength refers to the parameters and not to the specification of the activation function.

\begin{proposition} \label{prop:gnn_sin}
There exists a family $\Fc_b$ of $1$-dimensional, $1$-layer GNNs with bitlength $\Oc(b)$ 
such that $\vcdim_{\Gc}(\Fc_b)=b$.
\end{proposition}
\begin{proof}
The upper bound follows from standard learning-theoretic results for finite classes of GNNs where an upper bound in known logarithmic in the size of the class. The size of $\Fc_b$ is exponential in $b$ due to the bounded bitlength, feature width and layer depth, from which the result follows.

The above bounds for any such finite class $\Fc_b$. For the lower bound, we make this class explicit. First, for $\bx=(x_1,\ldots,x_b)\in\{0,1\}^b$ we define the parameter $\gamma_\bx=\sum\nolimits_{k=1}^{b}\frac{x_k}{4^k}$ and the activation function $a:\Rb\to\Rb:r\mapsto \sin(\frac{\pi}{2}r)$.
The graph invariant computed by the GNNs on a graph $G$ is defined as follows:

\begin{itemize}
    \item Initially, all vertices $v$ in $V(G)$ have the same $1$-dimensional feature, that is, $h_{\bx,v}^{(0)}=1$;
    \item The first layer computes the $1$-dimension feature $$h_{\bx,v}^{(1)} = a\Bigl(\gamma_\bx h_{\bx,v}^{(0)} + \gamma_\bx \textstyle\sum_{u\in N_G(v)}h_{\bx,u}^{(0)}\Bigr)=a\bigl((1+d_G(v))\gamma_\bx\bigr);$$ 
    \item The readout function sums the features computed in the first layer, that is, $\xi_{\gamma_\bx}(G)=\sum_{v\in V(G)}h_{\bx,v}^{(1)}$.
\end{itemize}
The class $\Fc_b$ consists of $2^b$ GNNs specified above, parameterized by $\bx\in\{0,1\}^b$.
We now define $b$ graphs that will be shattered by $\Fc_b$. For $j=1,2,\dots,b$ let $G_j$ be a graph consisting of the following three subgraphs:
\begin{itemize}
    \item a single result vertex denoted $r$;
    \item a $(4^j-4)$-regular graph of order $4^j-1$, with vertices denoted by $s_i$; and
    \item a $(4^{b+1}-4^j)$-regular graph of order $4^{b+1}-4^j+2$ with vertices denotes $t_i$.
\end{itemize}
These vertices are connected as follows:
\begin{itemize}
    \item Vertex $r$ is connected with all vertices of the second subgraph; and
    \item All vertices in second subgraph are connected with all vertices in the third subgraph.
\end{itemize}

It is easy to verify (and well-known) that the regular graphs used above exist. Indeed, this follows from the observation that their degree is strictly smaller than their order and the degree is even.  Table~\ref{tab:graph_summary_sin} summarizes the vertex types with their counts, degrees and connections.

\begin{table}[t]
    \centering
   
    \begin{tabular}{c|ccl}
         Vertex & Count & Degree & Connections  \\
         \hline
         $r$ & 1 & $4^j-1$ & all $s_i$ \\
         $s_i$ & $4^j-1$ & $1 + 4^j-4 + 4^{b+1}-4^j+2=4^{b+1}-1$ & $r$ and all $t_i$ \\
         $t_i$ & $4^{b+1}-4^j+2$ & $4^{b+1}-4^j + 4^j-1=4^{b+1}-1$ & all $s_i$
    \end{tabular} \caption{Summary of the different vertex types in graph $G_j$ using the proof of Proposition~\ref{prop:gnn_sin}.}\label{tab:graph_summary_sin}
\end{table}

We now run the GNN $\xi_{\gamma_\bx}$ on $G_j$ with, as mentioned before, $\gamma_\bx=\sum\nolimits_{k=1}^{b}\frac{x_k}{4^k}$ and find the following features
\begin{align*}
h_{\bx,t_i}^{(1)} = h_{\bx,s_i}^{(1)} &= a((1 + 4^{b+1} - 1) \gamma_\bx) =\sin\bigl(4^{b+1}\frac{\pi}{2}\gamma_\bx\bigr)= \sin(2k\pi)=0 \text{ with $k\in\Nb$, and}\\
h_{\bx,r}^{(1)} &= a((1 + 4^{j} -1) \gamma_\bx) = \sin\bigl(4^j\frac{\pi}{2}\gamma_\bx\bigr).
\end{align*}
This results in the following graph embedding
$$
\xi_{\gamma_\bx}(G_j) = \sin\bigl(4^j\frac{\pi}{2}\gamma_\bx\bigr).
$$
Due to Lemma~\ref{lemma:sin}, this quantitiy is larger than $\frac{2}{3}$ if and only if $x_j=1$ from which it follows that the $b$ graphs can be shattered by $\Fc_b$.
\end{proof}

\begin{proof}[Proof of Theorem~\ref{th:main}]
As the $\sin$ function is an analytic non-polynomial function, so is the function $a(r)\coloneq\sin(\frac{\pi}{2}r)$ and the GNNs in $\gnn_{\textsl{anp}}(1,L)$ use arbitrary precision reals, we can apply Proposition~\ref{prop:gnn_sin} with $b$ tending to infinity to achieve the result.

We can add additional layers without affecting the result by setting the first $L-1$ layers to $h_{\bx,v}^{(t)} = \sin\bigl(\frac{\pi}{2} h_{\bx,v}^{(t-1)}\bigr)$ and keeping the final layer as defined in Proposition~\ref{prop:gnn_sin}.
\end{proof}

\section{Conclusion}
In conclusion, our work provides important insights into the generalization properties of graph neural networks (GNNs) by examining their VC dimension. We extend existing results and demonstrate that even the most basic GNNs, including 1-dimensional GNNs with a single parameter and one layer, exhibit an infinite VC dimension when applied to unbounded graphs. This suggests that the ability of GNNs to generalize effectively may be inherently limited, particularly when the complexity of the graph structures increases. Our results cover highly expressive GNNs using analytical non-polynomial activation functions, and the less expressive GNNs with (piecewise) polynomial activation functions. All proofs are constructive.

\bibliography{main}

\end{document}